\documentclass[letterpaper]{article}
\pdfoutput=1
\usepackage{aaai16}
\usepackage{times}
\usepackage{helvet}
\usepackage{courier}
\nocopyright

\usepackage{amsmath}
\usepackage{amssymb}
\usepackage{amsthm}
\usepackage{enumitem}
\usepackage{mathrsfs}
\usepackage{microtype}
\usepackage{xspace}

\newtheorem{corollary}{Corollary}
\newtheorem{lemma}{Lemma}
\theoremstyle{definition}
\newtheorem{definition}{Definition}
\theoremstyle{plain}
\newtheorem{proposition}{Proposition}
\newtheorem{theorem}{Theorem}

\theoremstyle{definition}
\newtheorem{example}{Example}

\DeclareFontFamily{U}{MnSymbolC}{}
\DeclareSymbolFont{MnSyC}{U}{MnSymbolC}{m}{n}
\DeclareFontShape{U}{MnSymbolC}{m}{n}{
  <-6>  MnSymbolC5
  <6-7>  MnSymbolC6
  <7-8>  MnSymbolC7
  <8-9>  MnSymbolC8
  <9-10> MnSymbolC9
  <10-12> MnSymbolC10
  <12->   MnSymbolC12%
}{}

\newcommand{\powerset}[1]{2^{#1}}

\renewcommand{\k}{\mathit{k}}
\usepackage{tikz}
\usepackage{float}
\usetikzlibrary{shapes,decorations,shadows,matrix}
\tikzstyle{snode} = [draw,rounded corners,minimum size=4mm,font=\small,fill=gray!15,label distance=-1mm,node distance=0.1cm]

\newcommand{\attacks}{\rightarrowtail}
\newcommand{\AF}{\mathit{F}}
\newcommand{\AG}{\mathit{G}}
\newcommand{\AH}{\mathit{H}}

\newcommand{\pref}{\mathit{pr}}

\newcommand{\stage}{\mathit{stg}}
\newcommand{\comp}{\mathit{co}}
\newcommand{\ideal}{\mathit{id}}
\newcommand{\eager}{\mathit{eg}}

\newcommand{\cf}{\textit{cf}}
\newcommand{\cft}{\textit{cf2}}
\newcommand{\ad}{\textit{ad}}
\newcommand{\adm}{\textit{ad}}
\newcommand{\co}{\textit{co}}
\newcommand{\gr}{\textit{gr}}
\newcommand{\grd}{\gr}
\newcommand{\strad}{\textit{sad}}
\newcommand{\sad}{\textit{sad}}

\newcommand{\ground}{\gr}

\newcommand{\na}{\textit{na}}
\newcommand{\stb}{\textit{stb}}
\newcommand{\pr}{\textit{pr}}

\newcommand{\stg}{\textit{stg}}

\newcommand{\sta}{\textit{sta}}
\newcommand{\semi}{\textit{ss}}

\newcommand{\id}{\textit{id}}
\newcommand{\eg}{\textit{eg}}

\newcommand{\interm}[2]{$#1$-$#2$-intermediate}

\newcommand{\Ss}{\mathcal{S}}

\newcommand{\E}{S}

\newcommand{\sm}{\setminus}
\newcommand{\N}{\mathbb{N}}

\newcommand{\To}{\Rightarrow}
\newcommand{\oT}{\Leftarrow}
\newcommand{\ToT}{\Leftrightarrow}
\newcommand{\F}{\AF}
\newcommand{\G}{\AG}
\renewcommand{\H}{\AH}

\newcommand{\U}{\mathcal{U}}
\renewcommand{\r}{\mathfrak{r}}

\newcommand{\afs}{\mathscr{A}}

\title{Verifiability of Argumentation Semantics\thanks{This research has been supported by DFG (project BR 1817/7-1) and FWF (projects I1102 and P25521).}}
\author{Ringo Baumann\\
Leipzig University\\
Germany\\
\And
Thomas Linsbichler \and Stefan Woltran\\
TU Wien\\
Austria
}

\begin{document}

\maketitle

\begin{abstract}
\begin{quote}

Dung's abstract argumentation theory is a widely used formalism to model
conflicting information and to draw conclusions in such situations. Hereby,
the knowledge is represented by so-called argumentation frameworks (AFs) and
the reasoning is done via semantics extracting acceptable sets. All reasonable
semantics are based on the notion of conflict-freeness which means that
arguments are only jointly acceptable when they are not linked within the AF.
In this paper, we study the question which information on top of
conflict-free sets is needed to compute extensions of a semantics at hand.
We introduce a hierarchy of so-called verification classes specifying
the required amount of information. We show that well-known standard semantics
are exactly verifiable through a certain such class.  Our framework
also gives a means to study semantics lying inbetween known semantics,
thus contributing to a more abstract understanding of the different features
argumentation semantics offer.
\end{quote}
\end{abstract}

\section{Introduction}

In the late 1980s the idea of using \textit{argumentation}
to model nonmonotonic reasoning emerged (see \cite{Lou87,Pol87} as well as \cite{Prak01} for excellent overviews). 
Nowadays argumentation
theory is a vibrant subfield of 
Artificial Intelligence, covering aspects of
knowledge representation, multi-agent systems, and also philosophical
questions. Among other approaches which have been proposed for capturing representative
patterns of inference in argumentation theory \cite{strarg14}, Dung's abstract argumentation frameworks (AFs) \cite{dung95} play an important role within this research area. At the heart of Dung's approach lie the so-called \textit{argumentation semantics} (cf. \cite{baro11} for an excellent overview). Given an AF $\AF$, which is set-theoretically just a directed graph encoding arguments and attacks between them, a certain argumentation semantics $\sigma$ returns acceptable sets of arguments $\sigma(\AF)$, so-called
$\sigma$-\textit{extensions}. Each of these sets represents a reasonable position w.r.t.\ $\AF$ and $\sigma$. 

Over the last 20 years 
a series of abstract
argumentation semantics were introduced. The
motivations of these semantics range from the desired treatment of specific examples
to fulfilling a number of abstract principles. The comparison via abstract criteria of the different semantics available is a topic which emerged quite recently in the community (\cite{BaroniG07a} can be seen as the first paper in this line). 
Our work takes a further step towards a comprehensive understanding of argumentation semantics. In particular, we study the following question: Do we really need the entire AF $\AF$ to compute a certain argumentation semantics $\sigma$? In other words, is it possible to unambiguously determine acceptable sets  w.r.t.\ $\sigma$, given only partial information of the underlying framework $\AF$. In order to solve this problem let us start with the following reflections:

\begin{enumerate}
	\item As a matter of fact, one basic requirement of almost all existing semantics\footnote{See \cite{JakobovitsV99,Arieli12,GrossiM15} for exemptions.} is that of conflict-freeness, i.e. arguments within a reasonable position are not allowed to attack each other. Consequently, knowledge about conflict-free sets is an essential part for computing semantics.
	\item The second step is to ask the following: Which information on top on conflict-free sets has to be added? Imagine the set of conflict-free sets given by $\{\emptyset, \{a\}, \{b\}\}$. Consequently, there has to be at least one attack between $a$ and $b$. Unfortunately, this information is not sufficient to compute any standard semantics (except naive extensions, which are defined as $\subseteq$-maximal conflict-free sets) since we know nothing precise about the neighborhood of $a$ and $b$. The following three AFs possess exactly the mentioned conflict-free sets, but differ with respect to other 
	
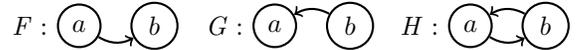
\begin{center}
\begin{tikzpicture}
	\node (a1) at (-0.6,0) [circle, thick, draw, label = left:$\F:$]{$a$};
    \node (b1) at (0.4,0) [circle, thick, draw]{$b$};
    
    \node (a2) at (2,0) [circle, thick, draw, label = left:$\G:$] {$a$};
    \node (b2) at (3.0,0) [circle, thick, draw] {$b$};
		
		\node (a3) at (4.6,0) [circle, thick, draw, label = left:$\H:$] {$a$};
    \node (b3) at (5.6,0) [circle, thick, draw] {$b$};

\draw[->,thick] (a1) to [thick, bend right] (b1);
\draw[->,thick] (b2) to [thick, bend right] (a2);
\draw[->,thick] (a3) to [thick, bend right] (b3);
\draw[->,thick] (b3) to [thick, bend right] (a3);
\end{tikzpicture}
\end{center}
	
	\item The final step is to try to minimize the added information. That is, which kind of knowledge about the neighborhood is somehow dispensable in the light of computation? Clearly, this will depend on the considered semantics. For instance, in case of stage semantics \cite{Verheij96}, which requests conflict-free sets of maximal range, we do not need any information about incoming attacks. This information can not be omitted in case of admissible-based semantics since incoming attacks require counterattacks.

\end{enumerate}

\noindent
The above considerations motivate the introduction of so-called verification classes specifying a certain amount of information. In a first step, we study the relation of these classes to each other. We therefore introduce the notion of being \textit{more informative} capturing the intuition that a certain class can reproduce the information of an other. We present a hierarchy w.r.t.\ this ordering. The hierarchy contains 15 different verification classes only. This is due to the fact that many syntactically different classes collapse to the same amount of information. 

We then formally define the essential property of a semantics $\sigma$ being \textit{verifiable} w.r.t.\ a certain verification class. We present a general theorem stating that any \textit{rational} semantics is exactly verifiable w.r.t.\ one of the $15$ different verification classes. Roughly speaking, a semantics is rational if attacks inbetween two self-loops can be omitted without affecting the set of extensions. An important aside hereby is that even the most informative class contains indeed less information than the entire framework by itself.  

In this paper we consider a representative set of %
standard semantics. All of them satisfy rationality and thus, are exactly verifiable w.r.t.\ a certain class. Since the theorem does not provide an answer to which verification class perfectly matches a certain rational semantics we study this problem one by one for any considered semantics. As a result, only $6$ different classes are essential to classify the considered standard semantics.

In the last part of the paper we study an application of the concept of verifiability. More precisely, we address the question of strong equivalence for semantics lying inbetween known semantics, so-called intermediate semantics. 
Strong equivalence is the natural counterpart to ordinary equivalence in monotonic theories
(see \cite{strong,Bau16} for abstract argumentation and
\cite{maher86,strongLP,strongcausal,strongnonmon} for other nonmonotonic theories). We provide characterization theorems relying on the notion of verifiability and thus, contributing to a more abstract understanding of the different features argumentation semantics offer. Besides these main results, we 
also give new characterizations for strong equivalence with respect to 
naive extensions and strong admissible sets.

\section{Preliminaries}
An \textit{argumentation framework} (AF) $\AF = (A,R)$ is a directed graph 
whose nodes $A \subseteq \U$
(with $\U$ being an infinite set of arguments, so-called \textit{universe})
are interpreted as \textit{arguments} and whose edges $R \subseteq A \times A$
represent \textit{conflicts} between them.
We assume that all AFs possess finitely\footnote{Finiteness of AFs is a common assumption in argumentation papers. A systematic study of the infinite case has begun quite recently (cf. \cite{BauS14} for an overview).} many arguments only and denote the collection of all AFs by $\afs$.
If $(a,b)\in R$ we say that $a$ \textit{attacks} $b$.
Alternatively, we write $a\attacks b$ as well as,
for some $S \subseteq A$,
$a \attacks S$ or $S\attacks b$ if there is some $c\in S$ attacked by $a$ or attacking $b$, respectively. 
An argument $a\in A$ is \textit{defended} by a set $S\subseteq A$
if for each $b\in A$ with $b\attacks a$, $S \attacks b$.
We define the \textit{range} of $S$ (in $\AF$)
as $S^+_F = S\cup \{a\mid S\attacks a\}$. Similarly, we use $S^-_\AF$
to denote the \textit{anti-range} of $S$ (in $\AF$) as
$S\cup \{a\mid a\attacks S\}$. 
Furthermore, we say that a set
$S$ is \textit{conflict-free} (in $\AF$) if there is no
argument $a \in S$ s.t. $S\attacks a$.
The set of all conflict-free sets of an AF $\AF$ is denoted by $\cf(\AF)$.
For an AF $\F = (B,S)$ we use $A(\F)$ and $R(\F)$ to refer to $B$ and $S$, respectively. 
Furthermore, we use $L(\AF) = \{a\mid (a,a)\in R(\AF)\}$ for the set of all self-defeating arguments. 
Finally, we introduce the union of AFs
$\AF$ and $\AG$ as $\AF\cup \AG = (A(\AF)\cup A(\AG),R(\AF)\cup R(\AG))$.

\subsection{Semantics}

A \textit{semantics} $\sigma$ assigns to each $\AF = (A,R)$ a set ${\sigma}(\AF)\subseteq\powerset{A}$ where the elements are called $\sigma$-\textit{extensions}. Numerous semantics are available. Each of them captures
different intuitions about how to reason about conflicting knowledge. 
We consider $\sigma\in\{\ad,\na,\stb,\pr,\co,\gr,\semi,\stg,\id,\eg\}$ for
admissible, naive, stable, preferred, complete, grounded, semi-stable, stage,
ideal, and eager semantics \cite{dung95,CaminadaCD12,Verheij96,DungMT07,Caminada07}.

\begin{definition}
\label{def:semantics}
      Given an AF $\AF = (A,R)$ and let $\E\subseteq A$.
      \begin{enumerate}
	\item { $\E\in\adm(\AF)$ iff $\E\in\cf(\AF)$ and each $a\in \E$ is defended by~$\E$,}
        \item { $\E\in\na(\AF)$ iff $\E\in\cf(\AF)$ and there is no $\E'\in\cf(\AF)$ s.t.\ $\E\subsetneq\E'$,}
	\item { $\E\in\stb(\AF)$ iff $\E\in\cf(\AF)$ and
	  $\E^+_\F=A$,}
	\item { $\E\in\pref(\AF)$ iff $\E\in\adm(\AF)$ and there is no $\E'\in\adm(\AF)$ s.t.\ $\E\subsetneq\E'$,}
	\item { $\E\in\comp(\AF)$ iff $\E\in\adm(\AF)$ and for
	  any $a\in A$ defended by~$\E$, %
          $a\in \E$,}
	\item { $\E\in\ground(\AF)$ iff $\E\in\comp(\AF)$ and
	  there is no $\E'\in\comp(\AF)$ s.t.\ $\E'\subsetneq\E$,}
	\item { $\E\in\semi(\AF)$ iff $\E\in\adm(\AF)$ and
	  there is no $\E'\in\adm(\AF)$ s.t.\ $\E^+_\F\subsetneq\E'^+_\F$,}
	\item { $\E\in\stg(\AF)$ iff $\E\in\cf(\AF)$ and 
	  there is no $\E'\in\cf(\AF)$ s.t.\ $\E^+_\F\subsetneq\E'^+_\F$,}
	\item { $\E\in\ideal(\AF)$ iff $\E\in\adm(\AF)$, $\E\subseteq
	  \bigcap\pref(\AF)$ and there is no $\E'\in\adm(\AF)$
	  satisfying $\E'\subseteq \bigcap\pref(\AF)$ s.t.~
	  $\E\subsetneq\E'$},
	\item { $\E\in\eager(\AF)$ iff $\E\in\adm(\AF)$, $\E\subseteq
	  \bigcap\semi(\AF)$ and there is no $\E'\in\adm(\AF)$
	  satisfying $\E'\subseteq \bigcap\semi(\AF)$ s.t.~
	  $\E\subsetneq~\E'$.}
      \end{enumerate}
    \end{definition}
	
For two semantics $\sigma$, $\tau$ we use $\sigma\subseteq\tau$ to indicate that
$\sigma(\AF)\subseteq\tau(\AF)$ for each AF $\AF \in \afs$.
If we have $\rho\subseteq\sigma$ and $\sigma\subseteq\tau$ for semantics $\rho,\sigma,\tau$,
we say that $\sigma$ is \textit{\interm{\rho}{\tau}}.
Well-known relations between semantics are
$\stb\subseteq\semi\subseteq\pref\subseteq\comp\subseteq\adm$,
meaning, for instance,
that $\semi$ is \interm{\stb}{\pref}.

\begin{definition}
\label{def:semantics_conditions}
We call a semantics $\sigma$
\emph{rational} if self-loop-chains are irrelevant.
That is, for every AF $\F$ it holds that $\sigma(\F) = \sigma(\F^l)$, where
$\F^l = (A(\F),
         R(\F) \sm \{(a,b) \in R(\F) \mid (a,a), (b,b) \in R(\F), a \neq b\})$.
\end{definition}

Indeed, all semantics introduced in Definition~\ref{def:semantics} are rational. 
A prominent semantics that is based on conflict-free sets, but is not rational 
is the $\cft$-semantics~\cite{BaroniGG05a}, since here chains of self-loops can have an 
influence on the SCCs of an AF (see also \cite{GagglW13}).

\subsection{Equivalence and Kernels}

The following definition captures the two main notions of equivalence available for non-monotonic formalisms, namely \textit{ordinary} (or \textit{standard}) \textit{equivalence} and \textit{strong} (or \textit{expansion}) \textit{equivalence}. A detailed overview of equivalence notion including their relations to each other can be found in \cite{zoo,zoo2}.
\begin{definition} Given a semantics $\sigma$. Two AFs $\AF$ and $\AG$ are 

\begin{itemize}
	\item \textit{standard equivalent} w.r.t.\ $\sigma$ ($\AF\equiv^\sigma\AG$) iff $\sigma(\AF) = \sigma(\AG)$,
	\item \textit{expansion equivalent} w.r.t.\ $\sigma$ ($\AF\equiv^\sigma_E\AG$) iff for all AFs~$\AH$: $\AF\cup\AH\equiv^\sigma\AG\cup\AH$ 
\end{itemize}

\end{definition}

Expansion equivalence can be decided syntactically via so-called \textit{kernels} \cite{strong}.
A kernel is a function $\k: \afs\mapsto\afs$ mapping each AF $\AF$ to another AF $\k(\AF)$
(which we may also denote as $\AF^\k$).
Consider the following definitions.

\begin{definition} \label{def:kernel}
Given an AF $\AF = (A,R)$ and a semantics $\sigma$. We define $\sigma$-kernels $\AF^{\k(\sigma)} = \left(A,R^{\k(\sigma)}\right)$ whereby

\begin{tabbing}
$R^{\k(\stb)}\! = R\ \sm \{(a, b) \mid$ \= $\ a \neq b, (a, a) \in R\}$,\\
 $R^{\k(\adm)} = R\ \sm \{(a, b) \mid a \neq b, (a, a) \in R,$\\
 \> $\{(b,a),(b,b)\}\cap R \neq\emptyset\}$,\\
 $R^{\k(\ground)} = R\ \sm \{(a, b) \mid a \neq b, (b, b) \in R,$\\ 
   \> $\{(a,a),(b,a)\}\cap R \neq\emptyset\}$,\\
$R^{\k(\comp)} = R\ \sm \{(a, b) \mid a \neq b, (a,a),(b,b)\in R \}$.
\end{tabbing}
\end{definition}

We say that a relation $\equiv\ \subseteq \afs\times\afs$ is \textit{characterizable through kernels} if there is a kernel $\k$, s.t. $\AF\equiv\AG$ iff $\AF^k=\AG^k$. Moreover, we say that a semantics $\sigma$ is \textit{compatible with a kernel} $k$ if $\F\equiv^{\sigma}_E\G$ iff $\F^k=\G^k$. All semantics (except naive semantics) considered in this paper are compatible with one of the four kernels introduced above. %
In the next section, we will complete these results taking naive semantics
and strong admissible sets into account.

\begin{theorem}{\cite{strong,equisurvey}} \label{the:strong} For any AFs $\F$ and $\G$,
\begin{enumerate}
	\item $\AF\equiv^{\sigma}_E \AG \ToT \AF^{\k(\sigma)} = \AG^{\k(\sigma)}$ with $\sigma\in\{\stb,\adm,\comp,\ground\}$,
	\item \mbox{$\AF\equiv^{\tau}_E \AG \ToT \AF^{\k(\adm)} = \AG^{\k(\adm)}$ with $\tau\in\{\pref,\ideal,\semi,\eager\}$,} 
	\item $\AF\equiv^{\stage}_E \AG \ToT \AF^{\k(\stb)} = \AG^{\k(\stb)}$. 
\end{enumerate}

\end{theorem}

\section{Complementing Previous Results}

In order to provide an exhaustive analysis of intermediate semantics (confer penultimate section) we provide missing kernels for naive semantics as well as strongly admissible sets. We start with the so-called \textit{naive kernel} characterizing expansion equivalence w.r.t.\ naive semantics. As an aside, the following kernel is the first one which adds attacks to the former attack relation. 

\begin{definition} \label{def:naivekernel}
Given an AF $\AF = (A,R)$. We define the \textit{naive kernel} $\AF^{k(\na)} = \left(A,R^{k(\na)}\right)$ whereby
$R^{k(\na)} = R\ \cup \left\{(a, b) \mid a \neq b, \{(a,a),(b,a),(b,b)\}\cap R \neq\emptyset\right\}.$
	\end{definition}

The following example illustrates the definition above.

\begin{example} Consider the AFs $\F$ and $\G$. Note that $\na(\F) = \na\left(\G\right) = \left\{\{a,c\},\{a,d\}\right\}$. Consequently, $\F \equiv^{\na} \G$.

\begin{center}
\begin{tikzpicture}
	\node (a) at (-0.6,0) [circle, thick, draw, label = left:$\F:$]{$a$};
    \node (b) at (0.6,0) [circle, thick, draw]{$b$};
    \node (c) at (1.8,0) [circle, thick, draw] {$c$};
    \node (d) at (3.0,0) [circle, thick, draw] {$d$};
		
		\node (a') at (-0.6,-1.0) [circle, thick, draw, label = left:$\G:$]{$a$};
    \node (b') at (0.6,-1.0) [circle, thick, draw]{$b$};
    \node (c') at (1.8,-1.0) [circle, thick, draw] {$c$};
    \node (d') at (3.0,-1.0) [circle, thick, draw] {$d$};

\draw[->,thick] (b) to [thick,loop,distance=0.5cm] (b);
\draw[->,thick] (d) to [thick, bend right] (c);
\draw[->,thick] (b') to [thick,loop,distance=0.5cm] (b');
\draw[->,thick] (c') to [thick, bend right] (d');
\draw[->,thick] (a') to [thick, bend right] (b');
\draw[->,thick] (b') to [thick, bend right] (c');

\end{tikzpicture}
\end{center}

In accordance with Definition~\ref{def:naivekernel} we observe that both AFs possess the same naive kernel $\H = \F^{\k(\na)} = \G^{\k(\na)}$. 

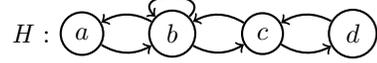
\begin{center}
\begin{tikzpicture}
	\node (a) at (-0.6,0) [circle, thick, draw, label = left:$\H:$]{$a$};
    \node (b) at (0.6,0) [circle, thick, draw]{$b$};
    \node (c) at (1.8,0) [circle, thick, draw] {$c$};
    \node (d) at (3.0,0) [circle, thick, draw] {$d$};

\draw[->,thick] (b) to [thick,loop,distance=0.5cm] (b);
\draw[->,thick] (d) to [thick, bend right] (c);
\draw[->,thick] (c) to [thick, bend right] (d);
\draw[->,thick] (a) to [thick, bend right] (b);
\draw[->,thick] (b) to [thick, bend right] (a);
\draw[->,thick] (b) to [thick, bend right] (c);
\draw[->,thick] (c) to [thick, bend right] (b);

\end{tikzpicture}
\end{center}
 The following theorem proves that possessing the same kernels is necessary as well as sufficient for being strongly equivalent, i.e. $\F \equiv^{\na}_{E} \G$.

\end{example}

\begin{theorem} For all AFs $\AF$,$\AG$, $$\AF\equiv^{\na}_E\AG \ToT \AF^{k(\na)}=\AG^{k(\na)}.$$
\end{theorem}

\begin{proof} In \cite{equisurvey} it was already shown that $\AF\equiv^{\na}_E\G$ iff jointly $A(\AF) = A(\G)$ and $\na(\AF)~=~\na(\G)$. Consequently, it suffices to prove that
$\AF^{k(\na)}=\G^{k(\na)}$ implies $A(\AF) = A(\G)$ as well as $\na(\AF)~=~\na(\G)$ and vice versa.

($\oT$) Given $\AF^{k(\na)}=\G^{k(\na)}$. By Definition~\ref{def:naivekernel} we immediately have $A(\AF) = A(\G)$. Assume now that $\na(\AF)\neq\na(\G)$ and without loss of generality let $S\in\na(\AF)\sm\na(\G)$. Obviously, for any AF $\H$, $\cf(\H) = \cf\left(\H^{\k(\na)}\right)$. Hence, there is an $S'$, s.t. $S\subsetneq S' \in \cf(\G)\sm\cf(\F)$. Thus, there are $a,b\in S'\sm S$, s.t. $(a,b)\in R(\F)\sm R(\G)$. Furthermore, $(a,a),(b,b)\notin R(\G)$ and since for any AF $\H$, $L(\H) = L\left(\H^{\k(\na)}\right)$ we obtain $(a,a),(b,b)\notin R(\F)$. Consequently, we have to consider $a\neq b$. Since $(a,b)\in R(\F)\sm R(\G)$, we obtain $(a,b),(b,a)\in R\left(\AF^{k(\na)}\right)$. Since $\AF^{k(\na)}=\G^{k(\na)}$ is assumed we derive $(a,b),(b,a)\in R\left(\AG^{k(\na)}\right)$. By Definition~\ref{def:naivekernel} we must have $(b,a)\in R(\G)$ contradicting the conflict-freeness of $S'$ in $\G$.

($\To$) We show the contrapositive, i.e. $\AF^{k(\na)}\neq\G^{k(\na)}$ implies $A(\AF) \neq A(\G)$ or $\na(\AF)~\neq~\na(\G)$. Observe that for any AF $\H$, $A(\H) = A\left(\H^{\k(\na)}\right)$. Consequently, if $A\left(\AF^{k(\na)}\right)\neq A\left(\G^{k(\na)}\right)$, then $A(\AF)\neq A(\G)$. Assume now $R\left(\AF^{k(\na)}\right)\neq R\left(\G^{k(\na)}\right)$. Without loss of generality let $(a,b)\in R\left(\AF^{k(\na)}\right)\sm R\left(\G^{k(\na)}\right)$. Since for any AF $\H$, $L(\H) = L\left(\H^{\k(\na)}\right)$ we obtain $a\neq b$. Furthermore, $(a,b)\in R\left(\AF^{k(\na)}\right)$ implies $\{(a,a),(a,b),(b,a),(b,b)\}\cap R\left(\AF\right) \neq \emptyset$ and consequently, for any $S\in\na(\F)$, $\{a,b\}\not\subseteq S$. Since $(a,b)\notin R\left(\G^{k(\na)}\right)$ we deduce $\{(a,a),(a,b),(b,a),(b,b)\}\cap R\left(\AF\right) = \emptyset$. Hence, $\{a,b\}\in\cf(\G)$ and thus, there exists a set $S\in\na(\G)$, s.t. $\{a,b\}\subseteq S$ (compare \cite[Lemma 3]{BauS14}) witnessing $\na(\AF)~\neq~\na(\G)$. 
\end{proof}

We turn now to \textit{strongly admissible sets} (for short, $\sad$) \cite{BaroniG07a}. 
We will show that,
beside grounded \cite{strong} and resolution based grounded semantics \cite{BaroniDG11,DvorakLOW14},
strongly admissible sets are characterizable through the grounded kernel. Consider the following self-referential definition
taken from \cite{Caminada14}.

\begin{definition} \label{def:strad} Given an AF $\AF = (A,R)$. A set $S\subseteq A$ is \textit{strongly admissible}, i.e. $S\in\strad(\AF)$ iff any $a\in S$ is defended by a strongly admissible set $S' \subseteq S\sm\{a\}$. 
\end{definition} 

The following properties are needed to prove the characterization theorem. The first two of them are already shown in \cite{BaroniG07b}. The third statement is an immediate consequence of the former.

\begin{proposition} \label{pro:strad:prop} Given two AFs $\AF$ and $\G$, then 
\begin{enumerate}
	\item $\gr(\AF)\subseteq\strad(\AF)\subseteq\ad(\AF)$,
	\item if $S\in\gr(\AF)$ we have: $S'\subseteq S$ for all $S'\in\strad(\AF)$, and
	\item $\strad(\AF) = \strad(\AG)$ implies $\gr(\AF) = \gr(\AG)$. 
\end{enumerate}
\end{proposition}

The following definition provides us with an alternative criterion for being a strong admissible set. In contrast to the former it allows one to construct strong admissible sets step by step. Thus, a construction method is given.  

\begin{definition} \label{def:stradnew} Given an AF $\AF = (A,R)$. A set $S\subseteq A$ is \textit{strongly admissible}, i.e. $S\in\strad(\AF)$ iff there are finitely many and pairwise disjoint sets $A_1,...,A_n$, s.t. $S = \bigcup_{1\leq i \leq n} A_i$ and $A_1\subseteq\Gamma_{\AF}(\emptyset)$\footnote{Hereby, $\Gamma$ is the so-called \textit{characteristic function} \cite{dung95} with $\Gamma_{\AF}(S) = \{a\in A\mid a \text{ is defended by } S \text{ in } \AF\}$. The term $\Gamma_{\AF}(\emptyset)$ can be equivalently replaced by $\{a\in A\mid a \text{ is unattacked} \}$.}  and furthermore, $\bigcup_{1\leq i \leq j} A_i \text{ defends } A_{j+1} \text{ for } 1\leq j \leq n-1$.

\end{definition} 

\begin{proposition} \label{pro:stradnew} Definitions~\ref{def:strad} and \ref{def:stradnew} are equivalent.
\end{proposition} 

\begin{proof} For the proof we use $S\in\strad_{k}(\AF)$ as a shorthand for $S\in\strad(\AF)$ in the sense of Definition~$k$.
$(\oT)$ Given $S\in\strad_{\ref{def:stradnew}}(\AF)$. Hence, there is a finite partition, s.t. $S = \bigcup_{1\leq i \leq n} A_i$, $A_1\subseteq\Gamma_{\AF}(\emptyset)$ and $\bigcup_{1\leq i \leq j} A_i \text{ defends } A_{j+1} \text{ for } 1\leq j \leq n-1$. Observe that $\bigcup_{1\leq i \leq j} A_i \in \strad_{\ref{def:stradnew}}(\AF)$ for any $j\leq n$. Let $a\in S$. Consequently, there is an index $i^*$, s.t. $a\in A_{i^*}$. Furthermore, since $\bigcup_{1\leq i \leq i^*-1} A_i \text{ defends } A_{i^*}$ by definition, we deduce that $\bigcup_{1\leq i \leq i^*-1} A_i\subseteq S\sm\{a\}$ defends $a$. We have to show now that (the smaller set w.r.t.\ $\subseteq$) $\bigcup_{1\leq i \leq i^*-1} A_i\in\strad_{\ref{def:strad}}(\AF)$. Note that $\bigcup_{1\leq i \leq i^*-1} A_i\in\strad_{\ref{def:stradnew}}(\AF)$.  Since we are dealing with finite AFs we may iterate our construction. Hence, no matter which elements are chosen we end up with a $\subseteq$-chain, s.t. $\emptyset\subseteq \bigcup_{1\leq i \leq i_e} A_i \subseteq S_e\sm{a_e} $ and $\emptyset$ defends $a_{e}$ for some index $i_e$, set $S_e$ and element $a_e$. This means, the question whether $S\in\strad_{\ref{def:strad}}(\AF)$ can be decided positively by proving  $\emptyset\in\strad_{\ref{def:strad}}(\AF)$. Since the empty set does not contain any elements we find $\emptyset\in\strad_{\ref{def:strad}}(\AF)$ concluding $\strad_{\ref{def:stradnew}}\subseteq\strad_{\ref{def:strad}}$.\\
$(\To)$ Given $S\in\strad_{\ref{def:strad}}(\AF)$, consider the following sets~$S_i$: $S_1 = \left(\Gamma(\emptyset)\sm \emptyset\right) \cap S$, $S_2 = \left(\Gamma(S_1)\sm S_1\right) \cap S$, $S_3 = \left(\Gamma(\bigcup_{i=1}^2 S_i)\sm \bigcup_{i=1}^2 S_i\right) \cap S$, \dots , $S_n = \left(\Gamma(\bigcup_{i=1}^{n-1} S_{i})\sm \bigcup_{i=1}^{n-1} S_{i}\right) \cap S$. Since we are dealing with finite AFs there has to be a natural $n\in\N$, s.t. $S_n = S_{n+1} = S_{n+2} = \dots$. Consider now the union of these sets, i.e. $\bigcup_{i=1}^{n} S_{i}$. We show now that $\bigcup_{i=1}^{n} S_{i}\in\strad_{\ref{def:stradnew}}(\AF)$ and $\bigcup_{i=1}^{n} S_{i} = S$. By construction we have $S_1 \subseteq \Gamma(\emptyset)$. Moreover, $\bigcup_{1\leq i \leq j} S_i \text{ defends } S_{j+1} \text{ for } 1\leq j \leq n-1$. This can be seen as follows. By definition $S_{j+1} = \left(\Gamma(\bigcup_{i=1}^{j} S_{i})\sm \bigcup_{i=1}^{j} S_{i}\right) \cap S$. This means, $S_{j+1} \subseteq \Gamma(\bigcup_{i=1}^{j} S_{i})$. Since $\Gamma(\bigcup_{i=1}^{j} S_{i})$ contains all elements defended by $\bigcup_{i=1}^{j} S_{i}$ we obtain $\bigcup_{i=1}^{n} S_{i}\in\strad_{\ref{def:stradnew}}(\AF)$. Obviously, $\bigcup_{i=1}^{j} S_{i} \subseteq S$. In order to derive a contradiction we suppose $S\not\subseteq \bigcup_{i=1}^{n} S_{i}$. This means there is a nonempty set $S^*$, s.t. $S = S^* \cup \bigcup_{i=1}^{n} S_{i}$. Let $S^* = \{s_1,\dots,s_k\}$. Observe that no element $s_i$ is defended by $\bigcup_{i=1}^{n} S_{i}$ (*). Since $S\in\strad_{\ref{def:strad}}(\AF)$ we obtain a set $S^*_1 \subseteq S\sm\{s_1\}$, s.t. $S^*_1\in\strad_{\ref{def:strad}}(\AF)$ and $S^*_1$ defends $s_1$. We now iterate this procedure ending up with a set $S^*_k \subseteq S^*_{k-1}\sm\{s_k\}\subseteq \bigcup_{i=1}^{n} S_{i}$, s.t. $S^*_k\in\strad_{\ref{def:strad}}(\AF)$ and $S^*_k$ defends $s_k$ contradicting (*) and concluding the proof. 
\end{proof}
The following example shows how to use the new construction method. 

\begin{example} Consider the following AF $\AF$.
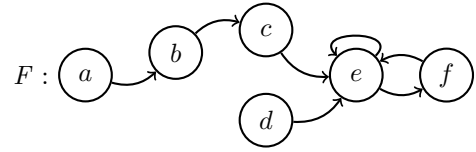
\begin{center}
\begin{tikzpicture}
	\node (a) at (-0.6,-0.6) [minimum size = 7mm, circle, thick, draw, label = left:$\AF:$]{$a$};
    \node (b) at (0.6,-0.3) [minimum size = 7mm, circle, thick, draw]{$b$};
    \node (c) at (1.8,0) [minimum size = 7mm, circle, thick, draw] {$c$};
		\node (d) at (1.8,-1.2) [minimum size = 7mm, circle, thick, draw] {$d$};
    \node (e) at (3.0,-0.6) [minimum size = 7mm, circle, thick, draw] {$e$};
		\node (f) at (4.2,-0.6) [minimum size = 7mm, circle, thick, draw] {$f$};

\draw[->,thick] (e) to [thick,loop,distance=0.5cm] (e);
\draw[->,thick] (a) to [thick, bend right] (b);
\draw[->,thick] (b) to [thick, bend left] (c);
\draw[->,thick] (c) to [thick, bend right] (e);
\draw[->,thick] (e) to [thick, bend right] (f);
\draw[->,thick] (f) to [thick, bend right] (e);
\draw[->,thick] (d) to [thick, bend right] (e);

\end{tikzpicture}
\end{center}
We have $\Gamma_{\AF}(\emptyset) = \{a,d\}$. Hence, for all $S\subseteq\{a,d\}$, $S\in\sad(\AF)$. Furthermore, $\Gamma_{\AF}(\{a\}) = \{a,c\}$, $\Gamma_{\AF}(\{d\}) = \{d,f\}$ and $\Gamma_{\AF}(\{a,d\}) = \{a,d,c,f\}$. This means, additionally $\{a,c\},\{d,f\},\{a,d,c\},\{a,d,f\},\{a,d,c,f\}\in\sad(\AF)$. Finally, $\Gamma_{\AF}(\{a,c\}) = \{a,c,f\}$ justifying the last missing set $\{a,c,f\}\in\sad(\AF)$. 
\end{example}

The following corollary is an immediate consequence of Definition~\ref{def:stradnew}. It is essential to prove the characterization theorem for strongly admissible sets.

\begin{corollary} \label{cor:strad}
Given an AF $\AF$ and two sets $B,B'\subseteq A(F)$. If $B$ defends $B'$, then $B\cup B'$ is strong admissible if $B$ is. 
\end{corollary}

The following lemma shows that the grounded kernel is insensitive w.r.t.\ strong admissible sets.

\begin{lemma} \label{lem:strad} For any AF $\AF$,  $\strad\left(\AF\right) = \strad\left(\AF^{k(\gr)}\right)$.
\end{lemma}
\begin{proof} The grounded kernel is node- and loop-preserving, i.e.\ $A(\AF) = A\left(\AF^{k(\gr)}\right)$ and $L(\AF)~=~L\left(\AF^{k(\gr)}\right)$. Furthermore, $\cf(\AF)\! =\! \cf\left(\AF^{k(\gr)}\right)$ and $\Gamma_{\AF}(\emptyset)~=~\Gamma_{\AF^{k(\gr)}}(\emptyset)$ as shown in \cite[Lemma~6]{strong}.\\
($\subseteq$) Given $S\in\strad\left(\F\right)$. The proof is by induction on $n$ indicating the number of  sets forming a suitable (according to Definition~\ref{def:stradnew}) partition of $S$. Let $n=1$. In consideration of the grounded kernel we observe $\Gamma_{\F}(\emptyset) = \Gamma_{\F^{k(\gr)}}(\emptyset)$, i.e. the set of unattacked arguments does not change. Since $S\subseteq\Gamma_{\F}(\emptyset)$ is assumed we are done. Assume now that the assertion is proven for any $k$-partition. Let $S$ be a $(k+1)$-partition, i.e. $S = \bigcup_{i=1}^{k+1} A_i$. According to induction hypothesis as well as Corollary~\ref{cor:strad} it suffices to prove $\bigcup_{i=1}^k A_i$ defends $A_{k+1}$ in $\F^{k(\gr)}$. Assume not, i.e. there are arguments $b\in A(\F)\sm S$, $c\in A_{k+1}$ s.t. $(b,c)\in R\left(\F^{k(\gr)}\right) \subseteq R(\AF)$ and for all $a\in \bigcup_{i=1}^k A_i$, $(a,b)\notin R\left(\F^{k(\gr)}\right)$ (*). Since $\bigcup_{i=1}^k A_i$ defends $A_{k+1}$ in $\F$ we deduce the existence of an argument $a\in\bigcup_{i=1}^k A_i$
 s.t. $(a,b)\in R\left(\F\right)$. Thus, $(a,b)$ is redundant w.r.t.\ the grounded kernel. According to
 Definition~\ref{def:kernel} and due to the conflict-freeness of $\bigcup_{i=1}^k A_i$ we have $(a,a)\notin R\left(\F\right)$ and $(b,a),(b,b)\in R\left(\F\right)$. Consequently, $(b,a)\in\F^{k(\gr)}$.
 Since $\bigcup_{i=1}^k A_i$ is a strong admissible $k$-partition in $\F$ we obtain by induction hypothesis that $\bigcup_{i=1}^k A_i$ is strong admissible in $\F^{k(\gr)}$ and therefore, admissible in $\F^{k(\gr)}$ (Proposition~\ref{pro:strad:prop}). Hence there has to be an argument $a\in \bigcup_{i=1}^k A_i$, s.t.\ $(a,b)\in R\left(\F^{k(\gr)}\right)$, contradicting~(*).

\noindent
($\supseteq$) Assume $S\in\strad\left(\F^{k(\gr)}\right)$. We show $S\in\strad\left(\F\right)$ by induction on $n$ indicating that $S$ is a $n$-partition in $\F^{k(\gr)}$. Due to $\Gamma_{\F}(\emptyset) = \Gamma_{\F^{k(\gr)}}(\emptyset)$ the base case is immediately clear. For the induction step let $S$ be a $(k+1)$-partition, i.e. $S = \bigcup_{i=1}^{k+1} A_i$. By induction hypothesis we may assume that $\bigcup_{i=1}^k A_i$ is strongly admissible in $\F$. Using Corollary~\ref{cor:strad} it suffices to prove $\bigcup_{i=1}^k A_i$ defends $A_{k+1}$ in $\F$. Assume not, i.e. there are arguments $b\in A(\F)\sm S$, $c\in A_{k+1}$ s.t. $(b,c)\in R\left(\F\right)$ and for all $a\in \bigcup_{i=1}^k A_i$, $(a,b)\notin R\left(\F\right)$. We even have $(a,b)\notin R\left(\F^{k(\gr)}\right)$ since $R\left(\F^{k(\gr)}\right)\subseteq R\left(\F\right)$. Consequently, $(b,c)$ has to be deleted in $\F^{k(\gr)}$. Definition~\ref{def:kernel} requires $(c,c)\in R\left(\F^{k(\gr)}\right)$ contradicting the conflict-freeness of $S$ in $\F^{k(\gr)}$.
\end{proof}

\begin{theorem} For any two AFs $\F$ and $\G$ we have,
$$\F\equiv^{\strad}_{E} \G \ToT \F^{k(\gr)} = \G^{k(\gr)}$$
\end{theorem}

\begin{proof} ($\To$) We show the contrapositive, i.e. $\F^{k(\gr)}~\neq~\G^{k(\gr)} \To \F\not\equiv^{\strad}_{E} \G$. Assuming $\F^{k(\gr)} \neq \G^{k(\gr)}$ implies $\F\not\equiv^{\gr}_{E} \G$ (Theorem~\ref{the:strong}). This means, there is an AF $\H$, s.t. $\gr(F\cup H) \neq \gr(G\cup H)$.
Due to statement 3 of Proposition~\ref{pro:strad:prop}, we deduce $\strad(F\cup H) \neq \strad(G\cup H)$ proving $\F\not\equiv^{\strad}_{E} \G$.

\noindent
($\oT$) Given $\F^{k(\gr)} = \G^{k(\gr)}$. Since expansion equivalence is a congruence w.r.t.\ $\cup$ we obtain $\left(\F\cup\H\right)^{k(\gr)} = \left(\G\cup\H\right)^{k(\gr)}$ for any AF $\H$. Consequently, $\strad\left(\left(\F\cup\H\right)^{k(\gr)}\right) = \strad\left(\left(\G\cup\H\right)^{k(\gr)}\right)$. Due to Lemma~\ref{lem:strad} we deduce $\strad(\F\cup\H) = \strad(\G\cup\H)$, concluding the proof. 
\end{proof}

\section{Verifiability}
\label{sec:verify}
 
In this section we study the question whether we really need
the entire AF $\F$ to compute the extensions of a given semantics.
Let us consider naive semantics.
Obviously, in order to determine naive extensions it suffices
to %
know
all conflict-free sets.
Conversely, knowing $\cf(\F)$ only does not allow to reconstruct $\F$ unambiguously.
This means, knowledge about $\cf(\F)$ is indeed less information than the entire AF by itself.
In fact, most of the existing semantics do not need information of the entire framework.
We will categorize the amount of information by taking the conflict-free sets as a basis
and distinguish between different amounts of knowledge about the neighborhood,
that is range and anti-range,
of these sets.

\begin{definition}
\label{def:classes}
We call a function $\r^x : \powerset{\U}\times \powerset{\U} \to \left(\powerset{\U}\right)^n$
($n>0$)
which is expressible via basic set operations only
\emph{neighborhood function}.
A neighborhood function $\r^x$ induces the \emph{verification class}
mapping each AF $\F$ to
\[
\widetilde{\F}^x = \{(S,\r^x(S^+_F,S^-_F)) \mid S \in \cf(\F)\}.
\]
\end{definition}

We coined the term neighborhood function because the induced verification classes apply these functions to the neighborhoods, i.e. range and anti-range of conflict-free sets. The notion of {\it expressible via basic set operations} simply means that (in case of $n = 1$) the expression $\r^x(A,B)$ is in the language generated by the following BNF: $$X ::= A\mid B \mid (X\cup X) \mid (X\cap X) \mid (X\setminus X).$$
Consequently, in case of $n=1$, we may distinguish eight
set theoretically different neighborhood functions, namely
\begin{align*} %
\r^\epsilon(S,S') &= \emptyset \\ %
\r^{+}(S,S') &= S \\ %
\r^{-}(S,S') &= S' \\ %
\r^{\mp}(S,S') &= S' \sm S \\ %
\r^{\pm}(S,S') &= S \sm S' \\ %
\r^{\cap}(S,S') &= S \cap S' \\
\r^{\cup}(S,S') &= S \cup S' \\
\r^{\Delta}(S,S') &= (S \cup S') \sm (S \cap S')
\end{align*}

A verification class encapsulates a certain amount of information about an AF, as the following example illustrates.

\begin{example}
\label{ex:verification_class}
Consider the following AF $\F$:
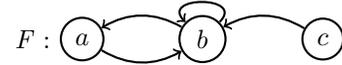
\begin{center}
\begin{tikzpicture}
	\node (a) at (-0.6,0) [circle, thick, draw, label = left:$\F:$]{$a$};
    \node (b) at (1,0) [circle, thick, draw]{$b$};

    \node (c) at (2.6,0) [circle, thick, draw] {$c$};

\draw[->,thick] (b) to [thick,loop,distance=0.5cm] (b);

\draw[->,thick] (b) to [thick, bend right] (a);
\draw[->,thick] (a) to [thick, bend right] (b);
\draw[->,thick] (c) to [thick, bend right] (b);
\end{tikzpicture}
\end{center}
Now take, for instance, the verification class induced by $\r^+$, that is
$\widetilde{\F}^+ = \{(S,\r^+(S^+_F,S^-_F))\mid S\in\cf(\F)\} = \{(S,S^+_F)\mid S\in\cf(\F)\}$,
storing information about conflict-free sets together with their associated ranges w.r.t.\ $\F$.
It contains the following tuples:
$(\emptyset,\emptyset)$, $(\{a\},\{b\})$, $(\{c\},\{b\})$, and $(\{a,c\},\{b\})$.
The verification class induced by $\r^\pm$ contains the same tuples %
but
$(\{a\},\emptyset)$ instead of $(\{a\},\{b\})$.
\end{example}

 Intuitively, it should be clear that the set $\widetilde{\F}^+$ suffices to compute stage extensions (i.e., range-maximal conflict-free sets) of~$\F$. This intuitive understanding of \textit{verifiability} will be formally specified in Definition~\ref{def:verifiable}.  
Note that a neighborhood function $\r^x$ may return $n$-tuples. Consequently, in consideration of the eight listed basic function we obtain (modulo reordering, duplicates, empty set) $2^7+1$ syntactically different neighborhood functions and therefore the same number of verification classes. As usual, we will denote the $n$-ary combination of basic functions $(\r^{x_1}(S,S'), \dots, \r^{x_n}(S,S'))$ as $\r^x(S,S')$ with $x=x_1\dots x_n$.

With the following definition we can put neighborhood functions into relation w.r.t.\ their information.
This will help us to show that actually many of the induced classes collapse to the same amount of information.

\begin{definition}
\label{def:informative}
Given neighborhood functions $\r^x$ and $\r^y$ returning $n$-tuples and $m$-tuples, respectively,
we say that $\r^x$ is \textit{more informative} than $\r^y$,
for short $\r^x \succeq \r^y$,
iff there is a function
$\delta : \left(\powerset{\U}\right)^n \to \left(\powerset{\U}\right)^m$
such that for any two sets of arguments $S,S' \subseteq \U$,
we have $\delta\left(\r^x(S,S')\right) = \r^y\left(S,S'\right)$.
\end{definition}

We will denote the strict part of $\succeq$ by $\succ$,
i.e.\ $\r^x \succ \r^y$ iff $\r^x \succeq \r^y$ and $\r^y \not\succeq \r^x$.
Moreover $\r^x \approx \r^y$ in case $\r^x \succeq \r^y$ and $\r^y \succeq \r^x$,
we say that $\r^x$ \emph{represents} $\r^y$ and vice versa.

\begin{figure}[t]
\centering
\begin{tikzpicture}
\matrix (a) [matrix of math nodes, column sep=0.41cm, row sep=0.5cm]{
& & & +- \\
+\pm & +\mp & \pm\mp & & \cap\cup & -\pm & -\mp \\[0.5cm]
+ & \pm & \cap & \Delta & \cup & \mp & - \\
& & & \epsilon \\};

\foreach \i/\j in {2-1/1-4, 2-2/1-4,  2-3/1-4, 2-5/1-4, 2-6/1-4, 2-7/1-4,
                   3-1/2-1, 3-1/2-2,
                   3-2/2-1, 3-2/2-3, 3-2/2-6,
                   3-3/2-1, 3-3/2-5, 3-3/2-7,
                   3-4/2-3, 3-4/2-5,
                   3-5/2-2, 3-5/2-5, 3-5/2-6,
                   3-6/2-2, 3-6/2-3, 3-6/2-7,
                   3-7/2-6, 3-7/2-7,
                   4-4/3-1, 4-4/3-2, 4-4/3-3, 4-4/3-4, 4-4/3-5, 4-4/3-6, 4-4/3-7}
         \draw[->] (a-\i) -- (a-\j);
\end{tikzpicture}
\caption{Representatives of neighborhood functions and their relation w.r.t.\ information;
         a node $x$ stands for the neighborhood function $\r^x$;
         an arrow from $x$ to $y$ means $\r^x \prec \r^y$.}
\label{fig:verification_classes}
\end{figure}
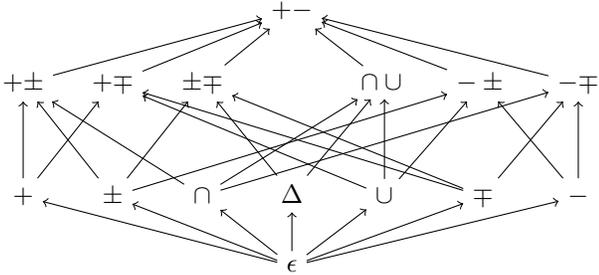

\begin{lemma}
\label{lemma:verification_classes}
All neighborhood functions are represented by the ones
depicted in Figure~\ref{fig:verification_classes}
and
the $\prec$-relation represented by arcs in Figure~\ref{fig:verification_classes} holds.
\end{lemma}
\begin{proof}
We begin by showing that all neighborhood functions are represented in Figure~\ref{fig:verification_classes}.
Clearly, each neighborhood function $\r^x$ represents itself, i.e. $\r^x \approx \r^x$.
All neighborhood functions for $n=1$
are are depicted in Figure~\ref{fig:verification_classes}.
We turn to $n=2$.
Consider the neighborhood functions $\r^{+\pm}$, $\r^{+\cap}$, and $\r^{\pm\cap}$,
defined as
$\r^{+\pm}(S,S') = (S,S \sm S')$,
$\r^{+\cap}(S,S') = (S,S \cap S')$, and
$\r^{\pm\cap}(S,S') = (S \sm S',S \cap S')$
for $S,S' \subseteq \U$.
Observe that $S = (S \sm S') \cup (S \cap S')$.
Hence, we can easily define functions in the spirit of Definition~\ref{def:informative}
mapping the images of the function to one another:
\begin{itemize}
\item
$\delta_1(\r^{+\pm}(S,S')) = \delta_1(S,S \sm S') =_{def} (S,S \sm (S \sm S')) = (S,S \cap S') = \r^{+\cap}(S,S')$;
\item
$\delta_2(\r^{+\cap}(S,S')) = \delta_2(S,S \cap S') =_{def} (S \sm (S \cap S'),S \cap S') = (S \sm S',S \cap S') = \r^{\pm\cap}(S,S')$;
\item
$\delta_3(\r^{\pm\cap}(S,S')) = \delta_3(S \sm S',S \cap S') =_{def} ((S \sm S') \cup (S \cap S'),S \sm S') = (S,S \sm S') = \r^{+\pm}(S,S')$.
\end{itemize}
Therefore, $\r^{+\pm} \approx \r^{+\cap} \approx \r^{\pm\cap}$.
In particular, they are all represented by $\r^{\pm}$.
We can apply the same reasoning to other combinations of neighborhood functions
and get the following equivalences w.r.t.\ information content:
$\r^{+\mp} \approx \r^{+\cup} \approx \r^{\mp\cup}$;
$\r^{\pm\mp} \approx \r^{\pm\Delta} \approx \r^{\mp\Delta}$;
$\r^{\cap\cup} \approx \r^{\cap\Delta} \approx \r^{\cup\Delta}$;
$\r^{-\pm} \approx \r^{-\cup} \approx \r^{\pm\cup}$; and
$\r^{-\mp} \approx \r^{-\cap} \approx \r^{\mp\cap}$,
with the functions stated first
acting as representatives in Figure~\ref{fig:verification_classes}.

For the remaining functions returning $2$-tuples we get
$\r^{+-} \approx \r^{+\Delta} \approx \r^{-\Delta}$
by %
\begin{itemize}
\item
$\delta_4(\r^{+-}(S,S')) = \delta_4(S,S') =_{def} (S,(S\cup S') \sm (S\cap S')) = \r^{+\Delta}(S,S')$;
\item
$\delta_5(\r^{+\Delta}(S,S')) = \delta_5(S,(S\cup S') \sm (S\cap S')) =_{def} ((S\sm((S\cup S') \sm (S\cap S')))\cup ((S\cup S') \sm (S\cap S'))\sm S,(S\cup S') \sm (S\cap S')) = (S',(S\cup S') \sm (S\cap S')) = \r^{-\cap}(S,S')$;
\item
$\delta_6(\r^{-\Delta}(S,S')) = \delta_6(S',(S\cup S') \sm (S\cap S')) =_{def} ((S'\sm((S\cup S') \sm (S\cap S')))\cup ((S\cup S') \sm (S\cap S'))\sm S',S') = (S,S') = \r^{+-}(S,S')$.
\end{itemize}

Finally, every neighborhood function $\r^{x_1 \dots x_n}$
with %
$n \geq 3$
is represented by $\r^{+-}$ since we can compute all possible sets
from $S$ and $S'$.

Now consider two functions $\r^x$ and $\r^y$ such that
there is an arrow from $x$ to $y$ in Figure~\ref{fig:verification_classes}.
It is easy to see that $\r^y \succeq \r^x$ since,
for sets of arguments $S$ and $S'$,
$\r^x(S,S')$ is either contained in $\r^y(S,S')$ or
obtainable from $\r^y(S,S')$ by basic set operations.
The fact that $\r^x  \not\succeq \r^y$, entailing $\r^y \succ \r^x$,
follows from the impossibility of finding a function $\delta$
such that $\delta(\r^x(S,S')) = \r^y(S,S')$.
\end{proof}

If the information provided by a neighborhood function 
is sufficient to compute the extensions, 
we say the semantics is verifiable by
the class induced by the neighborhood function.

\begin{definition} \label{def:verifiable}
A semantics $\sigma$ is \textit{verifiable} by
the verification class induced by
the neighborhood function
$\r^x$ returning $n$-tuples (or simply, $x$\textit{-veri\-fi\-able}) iff
there is a function (also called \textit{criterion})
$\gamma_{\sigma} : \left(\powerset{\U}\right)^n \times \powerset{\U} \to \powerset{\powerset{\U}}$
s.t. for every AF $\F \in \afs$ we have:
\[
\gamma_{\sigma}\left(\widetilde{\F}^x, A(\F)\right) = \sigma(\F).
\]
Moreover, $\sigma$ is \emph{exactly $x$-verifiable} iff
$\sigma$ is $x$-verifiable and there is no verification class induced by $\r^y$ with
$\r^y \prec \r^x$ such that $\sigma$ is $y$-verifiable.
\end{definition}

Observe that if a semantics $\sigma$ is $x$-verifiable then for any two AFs $\F$ and $\G$ with $\widetilde{\F}^x = \widetilde{\G}^x$ and $A(\F) = A(\G)$
it must hold that $\sigma(F) = \sigma(G)$.

We proceed with a list of criteria showing that any semantics mentioned in Definition~\ref{def:semantics} is verifiable by a verification class induced by a certain neighborhood function. %
In the following, we abbreviate the tuple $(\widetilde{\F}^x,A(\F))$ by $\widetilde{\F}^x_A$. 

{\small
\begin{align*}
\gamma_{\na}(\widetilde{\F}^\epsilon_A) =  \{ & S\mid 
           S\in\widetilde{\F}, S \textit{ is }\subseteq\textit{-maximal in }\widetilde{\F} \}; \\
\gamma_{\stg}(\widetilde{\F}^+_A) = \{ & S\mid 
           (S,S^+)\in\widetilde{\F}^+,
             S^+ \textit{ is }\subseteq\textit{-maximal in } \\
           & \{C^+\mid (C,C^+)\in\widetilde{\F}^+\} \}; \\
\gamma_{\stb}(\widetilde{\F}_A^+) = \{ & S\mid 
           (S,S^+)\in\widetilde{\F}^+, S^+ = A\}; \\
\gamma_{\ad}(\widetilde{\F}^\mp_A) = \{ & S\mid 
           (S,S^\mp)\in\widetilde{\F}^\mp, S^\mp = \emptyset\}; \\
\gamma_{\pr}(\widetilde{\F}^\mp_A) = \{ & S\mid 
           S\in\gamma_{\ad}(\widetilde{\F}^\mp_A),
             S \textit{ is }\subseteq\textit{-maximal in } \gamma_{\ad}(\widetilde{\F}^\mp_A) \}; \\
\gamma_{\semi}(\widetilde{\F}^{+\mp}_A) = \{ & S\mid
           S\in\gamma_{\ad}(\widetilde{\F}^\mp_A), 
             S^+ \textit{ is }\subseteq\textit{-maximal in } \\
           &     \{C^+\mid (C,C^+,C^\mp)\in\widetilde{\F}^{+\mp}, C\in\gamma_{\ad}(\widetilde{\F}^\mp_A)\} \}; \\
\gamma_{\id}(\widetilde{\F}^\mp_A) = \{ & S\mid
           S \textit{ is }\subseteq\textit{-maximal in } \\
           & \{C \mid C \in \gamma_{\ad}(\widetilde{\F}^\mp_A), C \subseteq \bigcap \gamma_\pr(\widetilde{\F}^\mp_A)\} \};\\
\gamma_{\eg}(\widetilde{\F}^{+\mp}_A) = \{ & S\mid
           S \textit{ is }\subseteq\textit{-maximal in } \\
           & \{C \mid C \in \gamma_{\ad}(\widetilde{\F}^\mp_A), C \subseteq \bigcap \gamma_\semi(\widetilde{\F}^{+\mp}_A)\}\};\\
\gamma_{\strad}(\widetilde{\F}^{-\pm}_A) = \{ & S \mid 
           (S,S^-,S^\pm) \in \widetilde{\F}^{-\pm}, \\
           & \exists (S_0,S_0^-,S_0^\pm),\dots,(S_n,S_n^-,S_n^\pm) \in \widetilde{\F}^{-\pm}: \\
           & (\emptyset = S_0 \subset \dots \subset S_n = S \wedge \\ 
           & \forall i \in \{1,\dots,n\} : S_i^- \subseteq S_{i-1}^\pm) \};
\end{align*}

\begin{align*}
\gamma_{\gr}(\widetilde{\F}^{-\pm}_A) = \{ & S \mid
           S \in \gamma_{\strad}(\widetilde{\F}^{-\pm}_A), \\
           & \forall (\bar{S},\bar{S}^-,\bar{S}^\pm) \in \widetilde{\F}^{-\pm} : 
             \bar{S} {\supset} S \Rightarrow (\bar{S}^- {\sm} S^\pm) {\neq} \emptyset) \}; \\
\gamma_{\co}(\widetilde{\F}^{+-}_A) = \{ & S \mid 
           (S,S^+,S^-) \in \widetilde{\F}^{+-}, (S^- \sm S^+) = \emptyset, \\
           & \forall (\bar{S},\bar{S}^+,\bar{S}^-) \in \widetilde{\F}^{+-} : 
             \bar{S} {\supset} S \Rightarrow (\bar{S}^- {\sm} S^+) {\neq} \emptyset) \}.
\end{align*}}

Instead of a formal proof we give the following explanations.
First of all it is easy to see that the naive semantics
is verifiable by the verification class induced by $\r^\epsilon$ 
since the naive extensions can be determined by the conflict-free sets.
Stable and stage semantics, on the other hand,
utilize the range of each conflict-free set in addition.
Hence they are verifiable by the verification class induced by $\r^+$.
Now consider admissible sets.
Recall that a conflict-free $S$ set is admissible if and only if
it attacks all attackers.
This is captured exactly by the condition $S^\mp = \emptyset$,
hence admissible sets are verifiable by the verification class induced by $\r^\mp$.
The same holds for preferred semantics,
since we just have to determine the maximal conflict-free sets
with $S^\mp = \emptyset$.
Semi-stable semantics, however, needs the range of each conflict-free set
in addition, see $\gamma_\semi$,
which makes it verifiable by the verification class induced by $\r^{+\mp}$.
Finally consider the criterion $\gamma_\co$.
The first two conditions for a set of arguments $S$ stand
for conflict-freeness and admissibility, respectively.
Now assume the third condition does not hold,
i.e., there exists a tuple
$(\bar{S},\bar{S}^+,\bar{S}^-) \in \widetilde{F}^{+-}$
with $\bar{S} \supset S$ and $\bar{S}^- \sm S^+ = \emptyset$.
This means that every argument attacking $\bar{S}$ is attacked
by $S$, i.e., $\bar{S}$ is defended by $S$.
Hence $S$ is not a complete extension,
showing that $\gamma_\co(\widetilde{F}^{+-}_A) = \co(F)$ for each $F \in \afs$.
One can verify that all criteria from the list are adequate
in the sense that they describe the extensions of the corresponding semantics.

We show now that the formal concepts of verifiability and being more informative behave correctly in the sense that the use of more informative neighborhood functions do not lead to a loss of verification capacity.

\begin{proposition}
If a semantics $\sigma$ is $x$-verifiable,
then $\sigma$ is verifiable by all verification classes induced by some $\r^y$ with $\r^y \succeq \r^x$.
\end{proposition}
\begin{proof}
As $\sigma$ is verifiable by the verification class induced by $\r^x$ it holds that
there is some $\gamma_\sigma$ such that for all $\F \in \afs$,
$\gamma_\sigma(\widetilde{\F}^x,A(\F)) = \sigma(\F)$.
Now let $\r^y \succeq \r^x$, meaning that there is some $\delta$ such that
$\delta(\r^y(S,S')) = \r^x$.
We define
$\gamma'_\sigma(\widetilde{\F}^y,A(\F)) = \gamma_\sigma(\{(S,\delta(\Ss)) \mid (S,\Ss) \in \widetilde{\F}^y\},A(\F))$ and
observe that
$\{(S,\delta(\Ss)) \mid (S,\Ss) \in \widetilde{\F}^y\} = \widetilde{\F}^x$,
hence
$\gamma'_\sigma(\widetilde{\F}^y,A(\F)) = \sigma(\F)$ for each $F \in \afs$.
\end{proof}

In order to prove unverifiability  of a semantics $\sigma$ w.r.t.\ a class induced by a certain $\r^x$
it suffices to present two AFs $\F$ and $\G$ such that
$\sigma(\F)\neq\sigma(\G)$
but, $\widetilde{\F}^x = \widetilde{\G}^x$ and $A(\F)~=~A(\G)$.
Then the verification class induced by $\r^x$ does not provide enough
information to verify $\sigma$. 

In the following we will use this strategy to show exact verifiability.
Consider a semantics $\sigma$ which is verifiable by a class induced by $\r^x$.
If $\sigma$ is unverifiable by all verifiability classes induced by $\r^y$ with
$\r^y \prec \r^x$
we have that $\sigma$ is exactly verifiable by $\r^x$.
The following examples study this issue for the
semantics under consideration.

\begin{example}
\label{ex:exact_verify}
The complete semantics is ${+-}$-verifiable
as seen before.
The following AFs show that it is even exactly verifiable by that class.
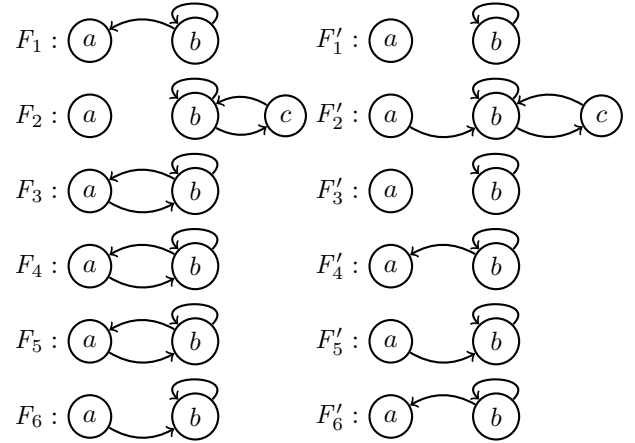
\begin{center}
\begin{tikzpicture}
	\node (a1) at (-0.6,0) [circle, thick, draw, label = left:$\F_1:$]{$a$};
    \node (b1) at (0.8,0) [circle, thick, draw]{$b$};

    \node (a1') at (3.4,0) [circle, thick, draw, label = left:$\F_1':$] {$a$};
    \node (b1') at (4.8,0) [circle, thick, draw] {$b$};

\draw[->,thick] (b1) to [thick,loop,distance=0.5cm] (b1);
\draw[->,thick] (b1') to [thick,loop,distance=0.5cm] (b1');

\draw[->,thick] (b1) to [thick, bend right] (a1);

	\node (a2) at (-0.6,-1) [circle, thick, draw, label = left:$\F_2:$]{$a$};
    \node (b2) at (0.8,-1) [circle, thick, draw]{$b$};
    \node (c2) at (2,-1) [circle, thick, draw]{$c$};

    \node (a2') at (3.4,-1) [circle, thick, draw, label = left:$\F_2':$] {$a$};
    \node (b2') at (4.8,-1) [circle, thick, draw] {$b$};
    \node (c2') at (6.2,-1) [circle, thick, draw] {$c$};

\draw[->,thick] (b2) to [thick,loop,distance=0.5cm] (b2);
\draw[->,thick] (b2') to [thick,loop,distance=0.5cm] (b2');

\draw[->,thick] (b2) to [thick, bend right] (c2);
\draw[->,thick] (c2) to [thick, bend right] (b2);
\draw[->,thick] (a2') to [thick, bend right] (b2');
\draw[->,thick] (c2') to [thick, bend right] (b2');
\draw[->,thick] (b2') to [thick, bend right] (c2');

	\node (a3) at (-0.6,-2) [circle, thick, draw, label = left:$\F_3:$]{$a$};
    \node (b3) at (0.8,-2) [circle, thick, draw]{$b$};

    \node (a3') at (3.4,-2) [circle, thick, draw, label = left:$\F_3':$] {$a$};
    \node (b3') at (4.8,-2) [circle, thick, draw] {$b$};

\draw[->,thick] (b3) to [thick,loop,distance=0.5cm] (b3);
\draw[->,thick] (b3') to [thick,loop,distance=0.5cm] (b3');

\draw[->,thick] (a3) to [thick, bend right] (b3);
\draw[->,thick] (b3) to [thick, bend right] (a3);

	\node (a4) at (-0.6,-3) [circle, thick, draw, label = left:$\F_4:$]{$a$};
    \node (b4) at (0.8,-3) [circle, thick, draw]{$b$};

    \node (a4') at (3.4,-3) [circle, thick, draw, label = left:$\F_4':$] {$a$};
    \node (b4') at (4.8,-3) [circle, thick, draw] {$b$};

\draw[->,thick] (b4) to [thick,loop,distance=0.5cm] (b4);
\draw[->,thick] (b4') to [thick,loop,distance=0.5cm] (b4');

\draw[->,thick] (a4) to [thick, bend right] (b4);
\draw[->,thick] (b4) to [thick, bend right] (a4);
\draw[->,thick] (b4') to [thick, bend right] (a4');

    \node (a5) at (-0.6,-4) [circle, thick, draw, label = left:$\F_5:$]{$a$};
    \node (b5) at (0.8,-4) [circle, thick, draw]{$b$};

    \node (a5') at (3.4,-4) [circle, thick, draw, label = left:$\F_5':$] {$a$};
    \node (b5') at (4.8,-4) [circle, thick, draw] {$b$};

\draw[->,thick] (b5) to [thick,loop,distance=0.5cm] (b5);
\draw[->,thick] (b5') to [thick,loop,distance=0.5cm] (b5');

\draw[->,thick] (a5) to [thick, bend right] (b5);
\draw[->,thick] (b5) to [thick, bend right] (a5);
\draw[->,thick] (a5') to [thick, bend right] (b5');

	\node (a6) at (-0.6,-5) [circle, thick, draw, label = left:$\F_6:$]{$a$};
    \node (b6) at (0.8,-5) [circle, thick, draw]{$b$};

    \node (a6') at (3.4,-5) [circle, thick, draw, label = left:$\F_6':$] {$a$};
    \node (b6') at (4.8,-5) [circle, thick, draw] {$b$};

\draw[->,thick] (b6) to [thick,loop,distance=0.5cm] (b6);
\draw[->,thick] (b6') to [thick,loop,distance=0.5cm] (b6');

\draw[->,thick] (b6') to [thick, bend right] (a6');
\draw[->,thick] (a6) to [thick, bend right] (b6);

\end{tikzpicture}
\end{center}
First consider the AFs $\F_1$ and $\F_1'$,
and observe that
$\widetilde{\F_1}^{+\pm} = \{(\emptyset,\emptyset,\emptyset),(\{a\},\emptyset,\emptyset)\} = \widetilde{\F_1'}^{+\pm}$.
On the other hand $\F_1$ and $\F_1'$
differ in their complete extensions
since $\co(\F_1) = \{\emptyset\}$ but
$\co(\F_1') = \{\{a\}\}$.
Therefore complete semantics is unverifiable
by the verification class induced by $\r^{+\pm}$.
Likewise, this can be shown for the classes induced by
$\r^{-\mp}$, $\r^{\pm\mp}$, $\r^{-\pm}$, $\r^{+\mp}$,
and $\r^{\cap\cup}$, respectively:

{\small
\begin{itemize}\itemsep0pt
\item %
$\widetilde{\F_2}^{-\mp} = \{(\emptyset,\emptyset,\emptyset),$ $(\{a\},\emptyset,\emptyset),$ $(\{a,c\},\{b\},\emptyset),$ $(\{c\},\{b\},\emptyset)\} = \widetilde{\F_2'}^{-\mp}$,
but
$\co(\F_2) = \{\{a\},\{a,c\}\} \neq \{\{a,c\}\} = \co(\F_2')$.
\item
$\widetilde{\F_3}^{\pm\mp} = %
\widetilde{\F_3'}^{\pm\mp}$,
but
$\co(\F_3) = \{\emptyset,\{a\}\} \neq \{\{a\}\} = \co(\F_3')$.
\item
$\widetilde{\F_4}^{-\pm} = %
\widetilde{\F_4'}^{-\pm}$,
but
$\co(\F_4) = \{\emptyset,\{a\}\} \neq \{\emptyset\} = \co(\F_4')$.
\item
$\widetilde{\F_5}^{+\mp} = %
\widetilde{\F_5'}^{+\mp}$,
but
$\co(\F_5) = \{\emptyset,\{a\}\} \neq \{\{a\}\} = \co(\F_5')$.
\item
$\widetilde{\F_6}^{\cap\cup} = %
\widetilde{\F_6'}^{\cap\cup}$,
but
$\co(\F_6) = \{\{a\}\} \neq \{\emptyset\} = \co(\F_6')$.
\end{itemize}}
Hence the complete semantics
is exactly verifiable by the verification class induced by $\r^{+-}$.
\end{example}

\begin{example}
\label{ex:semi_eager_unverifiable}
Consider the semi-stable and eager semantics
and recall that they are ${+\mp}$-verifiable
In order to show exact verifiability it suffices to show
unverifiability by the classes induced by $\r^+$, $\r^\cup$, and $\r^\mp$
(cf.\ Figure~\ref{fig:verification_classes});
$F_1$ and $F_6$ are taken from Example~\ref{ex:exact_verify} above.
{\small
\begin{itemize}\itemsep0pt
\item
$\widetilde{\F_1}^{+} = %
\widetilde{\F_1'}^{+}$,
but
$\semi(\F_1) = \eager(\F_1) = \{\emptyset\} \neq \{\{a\}\} = \semi(\F_1') = \eager(\F_1')$.
\item
$\widetilde{\F_6}^{\cup} = %
\widetilde{\F_6'}^{\cup}$,
but
$\semi(\F_6) = \eager(\F_6) = \{\{a\}\} \neq \{\emptyset\} = \semi(\F_6') = \eager(\F_6')$.
\item
$\widetilde{\F_7}^{\mp} = %
\widetilde{\F_7'}^{\mp}$,
but
$\semi(\F_7) = \{\{b\}\} \neq \{\{a\},\{b\}\} = \semi(\F_7')$ and
$\eager(\F_7) = \{\{b\}\} \neq \{\emptyset\} = \eager(\F_7')$.
\end{itemize}}
\begin{center}
\begin{tikzpicture}
	\node (a1) at (-0.6,0) [circle, thick, draw, label = left:$\F_7:$]{$a$};
    \node (b1) at (0.8,0) [circle, thick, draw]{$b$};
    \node (c1) at (2,0) [circle, thick, draw]{$c$};

    \node (a1') at (3.4,0) [circle, thick, draw, label = left:$\F_7':$] {$a$};
    \node (b1') at (4.8,0) [circle, thick, draw] {$b$};
    \node (c1') at (6,0) [circle, thick, draw] {$c$};

\draw[->,thick] (c1) to [thick,loop,distance=0.5cm] (c1);
\draw[->,thick] (c1') to [thick,loop,distance=0.5cm] (c1');

\draw[->,thick] (b1) to [thick, bend right] (a1);
\draw[->,thick] (a1) to [thick, bend right] (b1);
\draw[->,thick] (b1') to [thick, bend right] (a1');
\draw[->,thick] (a1') to [thick, bend right] (b1');
\draw[->,thick] (b1) to [thick, bend right] (c1);
\end{tikzpicture}
\end{center}
Hence, both the semi-stable and eager semantics
are exactly verifiable by the verification class induced by $\r^{+\mp}$.
\end{example}

\begin{example}
\label{ex:grd_sad_unverifiable}
Now consider the grounded and strong admissible semantics
and recall that they are ${-\pm}$-verifiable
In order to show exact verifiability we have to show
unverifiability by the classes induced by $\r^\pm$, $\r^-$, and $\r^\cup$
(cf.\ Figure~\ref{fig:verification_classes}); again,
the AFs from Example~\ref{ex:exact_verify} can be reused.
{\small
\begin{itemize}\itemsep0pt
\item
$\widetilde{\F_1}^{\pm} = \widetilde{\F_1'}^{\pm}$,
but
$\grd(\F_1) = \{\emptyset\} \neq \{\{a\}\} = \grd(\F_1')$ and
$\sad(\F_1) = \{\emptyset\} \neq \{\emptyset,\{a\}\} = \sad(\F_1')$.
\item
$\widetilde{\F_2}^{-} = \widetilde{\F_2'}^{-}$,
but
$\grd(\F_2) = \{\{a\}\} \neq \{a,c\} = \grd(\F_2')$ and
$\sad(\F_2) = \{\emptyset,\{a\}\} \neq \{\emptyset,\{a\},\{a,c\}\} = \sad(\F_2')$
\item
$\widetilde{\F_6}^{\cup} = \widetilde{\F_6'}^{\cup}$,
but
$\grd(\F_6) = \{\{a\}\} \neq \{\emptyset\} = \grd(\F_6')$ and
$\sad(\F_6) = \{\emptyset,\{a\}\} \neq \{\emptyset\} = \sad(\F_6')$.
\end{itemize}}
Hence, both the grounded and strong admissible semantics
are exactly verifiable by the verification class induced by $\r^{+\mp}$.
\end{example}

\begin{example}
\label{ex:last_unverifiable}
Finally consider stable, stage, admissible, preferred and ideal semantics.
They are either $+$-verifiable ($\stb$ and $\stage$) or $\mp$-verifiable ($\adm$, $\pref$, and $\ideal$).
In order to show that these verification classes are exact
we have to show unverifiability w.r.t.\ the verification class induced by $\r^\epsilon$.
Consider, for instance, the AFs $\F_4$ and $\F_4'$ from Example~\ref{ex:exact_verify}.
We have
$\widetilde{\F_4}^{\epsilon} = \widetilde{\F_4'}^{\epsilon}$,
but $\adm(\F_4) = \{\emptyset,\{a\}\} \neq  \{\emptyset\} = \adm(\F_4')$,
$\stb(\F_4) = \{\{a\}\} \neq  \emptyset = \stb(\F_4')$, and
$\sigma(\F_4) = \{\{a\}\} \neq \{\emptyset\} = \sigma(\F_4')$ for
$\sigma \in \{\stage,\pref,\ideal\}$,
showing exactness of the respective verification classes.
\end{example}

The insights obtained through Examples~\ref{ex:exact_verify},~\ref{ex:semi_eager_unverifiable},~\ref{ex:grd_sad_unverifiable}, and~\ref{ex:last_unverifiable}
show that the
verification classes obtained
from the criteria given above
are indeed exact.
Figure~\ref{fig:classes_of_semantics} shows the relation between the semantics
under consideration
with respect to their exact verification classes.

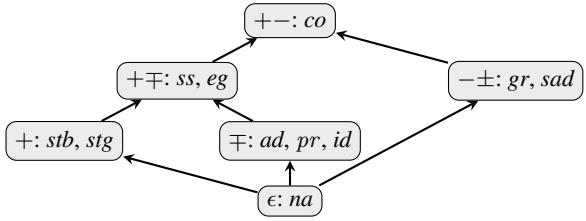
\begin{figure}[t]
\centering
\begin{tikzpicture}
\path
node[snode] (na) {$\epsilon$: $\na$}
++(-3,0.8) node[snode] (stb) {$+$: $\stb$, $\stg$} 
++(3,0) node[snode] (pref) {$\mp$: $\adm$, $\pref$, $\id$}
++(-1.5,0.8) node[snode] (semi) {$+\mp$: $\semi$, $\eg$}
++(4.5,0) node[snode] (grd) {$-\pm$: $\gr$, $\sad$}   
++(-3,0.8) node[snode] (co) {$+-$: $\co$};
				\path [->, thick,>=stealth]
			(na) edge (stb)
            (na) edge (pref)
            (na) edge (grd)
            (pref) edge (semi)
            (stb) edge (semi)
            (semi) edge (co)
            (grd) edge (co)
			;	
\end{tikzpicture}
\caption{Semantics and their exact verification classes.}
\label{fig:classes_of_semantics}
\end{figure}

We turn now to the main theorem stating that any \textit{rational} semantics (recall that all semantics we consider in this paper are rational) is exactly verifiable by one of the $15$ different verification classes. %

\begin{theorem}
Every semantics which is %
rational
is exactly verifiable by a verification class induced by
one of the neighborhood functions
presented in Figure~\ref{fig:verification_classes}.
\end{theorem}
\begin{proof}
First of all note that by Lemma~\ref{lemma:verification_classes},
$\r^{\epsilon}$ is the least informative neighborhood function and for every other
neighborhood function $\r^x$ %
it holds that $\r^\epsilon \preceq \r^{-}$.
Therefore, if a semantics is verifiable by the verification class induced by any $\r^x$
then it is exactly verifiable by a verification class induced by some $\r^y$
with $\r^\epsilon \preceq \r^y \preceq \r^x$.
Moreover, if a semantics is exactly verifiable by a class, then it is by definition also verifiable by this class.
Hence it remains to show that every semantics which is rational is verifiable by a verification class
presented in Figure~\ref{fig:verification_classes}.

We show the contrapositive, i.e.,
if a semantics is not verifiable by a verification class
induced by one of the neighborhood functions
presented in Figure~\ref{fig:verification_classes}
then it is not rational.

Assume a semantics $\sigma$ is not verifiable by one of the verification classes.
This means $\sigma$ is not verifiable by the verification class induced by $\r^{+-}$.
Hence there exist two AFs $\F$ and $\G$ such that
$\widetilde{\F}^{+-} = \widetilde{\G}^{+-}$ and
$A(\F) = A(\G)$,
but $\sigma(\F) \neq \sigma(\G)$.
For every argument $a$ which is not self-attacking,
a tuple $(\{a\},\{a\}^+,\{a\}^-)$ is contained in $\widetilde{\F}^{+-}$ (and in $\widetilde{\G}^{+-}$).
Hence $\F$ and $\G$ have the same not-self-attacking arguments and,
moreover these arguments have the same ingoing and outgoing attacks in $\F$ and $\G$.
This, together with $A(\F)=A(\G)$ implies that $\F^l = \G^l$ (see Definition~\ref{def:semantics_conditions}) holds.
But since $\sigma(\F) \neq \sigma(\G)$
we get that $\sigma$ is not rational,
which was to show.
\end{proof}

Note that the criterion giving evidence for verifiability
of a semantics by a certain class
has access to the set of arguments of a given framework.
In fact, only the criterion for stable semantics makes use of that.
Indeed, stable semantics needs this information since it is not verifiable by any class when using
a weaker notion of verifiability, which rules out the usage of $A(\F)$.

\section{Intermediate Semantics} \label{sec:intermediate}
A type of semantics which has aroused quite some interest in the literature (see e.g.\ \cite{BaroniG07b} and \cite{NievesOZ11})
are intermediate semantics,
i.e.\ semantics which %
yield results
lying
between two existing semantics. The introduction of $\sigma$-$\tau$-intermediate semantics can be motivated by deleting 
\textit{undesired} (or add \textit{desired}) $\tau$-extensions while 
guaranteeing all reasonable positions w.r.t.\ $\sigma$. 
In other words, \interm{\sigma}{\tau} semantics can be seen 
as sceptical or credulous acceptance shifts within the range of $\sigma$ and $\tau$. 

A natural question is whether we can make any statements
about compatible kernels of intermediate semantics.
In particular, if semantics $\sigma$ and $\tau$ are compatible with some kernel $\k$,
is then every \interm{\sigma}{\tau} semantics $\k$-compatible.
The following example answers this question negatively.

\begin{example}
\label{ex:stagle}
Recall from Theorem~\ref{the:strong} that both
stable and stage semantics are compatible with $k(\stb)$,
i.e.\ $\F\equiv^{\stb}_E\G \ToT \F\equiv^{\stg}_E~\G \ToT  \F^{k(\stb)} = \G^{k(\stb)}$.
Now we define the following \interm{\stb}{\stage} semantics, say \textit{stagle} semantics:
Given an AF $\F = (A,R)$, $S\in\sta(\F)$ iff
$S\in \cf(\F)$, $S^+_\F \cup S^-_\F = A$ and
for every $T\in cf(\F)$ we have $S^+_\F \not\subset T^+_\F$.
Obviously, it holds that $\stb\subseteq\sta\subseteq\stg$ and 
$\stb\neq\sta$ as well as $\sta\neq\stg$,
as witnessed by the following AF $\F$:
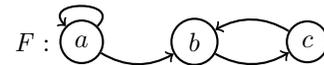
\begin{center}
\begin{tikzpicture}
    \node (A) at (1.0,0.0) [circle, thick, draw, label = left:$\F:$]{$a$};
    \node (B) at (2.5,0.0) [circle, thick, draw]{$b$};
    \node (C) at (4.0,0.0) [circle, thick, draw]{$c$};
		
\draw[->,thick] (A) to [thick,loop,distance=0.5cm] (A);
\draw[->,thick] (B) to [thick,bend right] (C);
\draw[->,thick] (C) to [thick,bend right] (B);
\draw[->,thick] (A) to [thick,bend right] (B);
\end{tikzpicture}
\end{center}
It is easy to verify that $\stb(\F) = \emptyset \subset \sta(\F) = \{\{b\}\} \subset \stg(\F) = \{\{b\},\{c\}\}$.
We proceed by showing that stagle semantics is not compatible with $k(\stb)$.
To this end consider $\F^{k(\stb)}$, which is depicted below.
\begin{center}
\begin{tikzpicture}
    \node (A) at (1.0,0.0) [circle, thick, draw, label = left:$\F^{k(\stb)}:$]{$a$};
    \node (B) at (2.5,0.0) [circle, thick, draw]{$b$};
    \node (C) at (4.0,0.0) [circle, thick, draw]{$c$};

\draw[->,thick] (A) to [thick,loop,distance=0.5cm] (A);
\draw[->,thick] (B) to [thick,bend right] (C);
\draw[->,thick] (C) to [thick,bend right] (B);

\end{tikzpicture}
\end{center}
Now, $\sta\left(\F^{k(\stb)}\right) = \{\{b\},\{c\}\}$ witnesses $\F\not\equiv^{\sta}\F^{k(\stb)}$ and therefore, $\F\not\equiv^{\sta}_E\F^{k(\stb)}$. Since $\F^{k(\stb)} = \left(\F^{k(\stb)}\right)^{k(\stb)}$ we are done, i.e. stagle semantics is indeed not compatible with the stable kernel.

\end{example}

It is the main result of this section that compatibility of intermediate semantics w.r.t.\ a certain kernel can be guaranteed if verifiability w.r.t.\ a certain class is presumed. The provided characterization theorems generalize former results presented in \cite{strong}. Moreover, due to the abstract character of the theorems the results are applicable to semantics which may be defined in the future.

Before turning to the characterization theorems we state some implications of verifiability. %
In particular, under the assumption that $\sigma$ is verifiable by a certain class, equality of certain kernels 
implies expansion equivalence w.r.t.\ $\sigma$.

\begin{proposition}
\label{pro:cf-range-verifiable}
For any $+$-verifiable semantics $\sigma$ we have
\[
\F^{k(\stb)} = \G^{k(\stb)} \To \F\equiv^\sigma_E\G.
\]
\end{proposition}

\begin{proof}
In \cite{strong} it was shown that 
\mbox{$\F^{k(\stb)} = \G^{k(\stb)} \To (\F\cup\H)^{k(\stb)} = (\G\cup\H)^{k(\stb)}$} (i).
Consider now a $+$-verifiable semantics $\sigma$. In order to show 
$\sigma\left(\F\right) = \sigma\left(\F^{k(\stb)}\right)$ (ii)
we prove $\widetilde{\F}^+ = \widetilde{\F^{k(\stb)}}^+$ (*) first.
It is easy to see that $S\in\cf(\F)$ iff $S\in\cf\left(\F^{k(\stb)}\right)$.
Furthermore, since $k(\stb)$ deletes an attack $(a,b)$ only if $a$ is self-defeating
we deduce that ranges does not change as long as conflict-free sets are considered.
Thus, $\sigma(\F) =_{\text{ \tiny (Def.)}} \gamma_{\sigma}(\widetilde{\F}^+) =_{\text{ \tiny (*)}} \gamma_{\sigma}(\widetilde{\F^{k(\stb)}}^+) =_{\text{ \tiny (Def.)}} \sigma(\F^{k(\stb)})$. 

Now assume that $\F^{k(\stb)} = \G^{k(\stb)}$ and
let $S\in\sigma(\F\cup~\H)$ for some AF $\H$.
We have to show that $S\in\sigma(\G\cup\H)$.
Applying (ii) we obtain $S\in\sigma\left((\F\cup\H)^{k(\stb)}\right)$. Furthermore, using (i) we deduce
$S\in\sigma\left((\G\cup\H)^{k(\stb)}\right)$.
Finally, $S\in\sigma\left(\G\cup\H\right)$ by applying (ii),
which concludes the proof.
\end{proof}

The following results can be shown in a similar manner. 

\begin{proposition}
\label{pro:cf-range-?-verifiable}
For any $+\mp$-verifiable semantics $\sigma$ we have
\[
\F^{k(\ad)} = \G^{k(\ad)} \To \F\equiv^\sigma_E\G.
\]
\end{proposition}

\begin{proposition}
\label{pro:cf-range-inrange-verifiable}
For any $+-$-verifiable semantics $\sigma$ we have
\[
\F^{k(\co)} = \G^{k(\co)} \To \F \equiv^\sigma_E \G.
\]
\end{proposition}

\begin{proposition}
\label{pro:cf-inrange-pm-verifiable}
For any $-\pm$-verifiable semantics $\sigma$ we have
\[
\F^{k(\gr)} = \G^{k(\gr)} \To \F\equiv^\sigma_E\G.
\]
\end{proposition}

\begin{proposition}
\label{pro:cf-verifiable}
For any $\epsilon$-verifiable semantics $\sigma$ we have
\[
\F^{k(\na)} = \G^{k(\na)} \To \F\equiv^\sigma_E\G.
\]
\end{proposition}

We proceed with general characterization theorems. The first one states that \interm{\stb}{\stage} semantics are compatible with stable kernel if $+$-verifiability is given. Consequently, stagle semantics as defined in Example~\ref{ex:stagle} can not be $+$-verifiable. 

\begin{theorem}
\label{the:stablefirst}
Given a semantics $\sigma$ which is $+$-verifiable and \interm{\stb}{\stg}, it holds that
\[
\F^{k(\stb)} = \G^{k(\stb)} \ToT \F\equiv^\sigma_E\G.
\]
\end{theorem}
\begin{proof}
($\To$) Follows directly from Proposition~\ref{pro:cf-range-verifiable}.

\noindent
($\oT$)
We show the contrapositive, i.e.\ $\F^{k(\stb)} \neq \G^{k(\stb)} \To \F\not\equiv^\sigma_E\G$.
Assuming $\F^{k(\stb)} \neq \G^{k(\stb)}$ implies $\F \not\equiv^\stg_E \G$,
i.e.\ there exists an AF $\H$ such that $\stg(\F \cup \H) \neq \stg(\G \cup \H)$ and therefore, $\stb(\F \cup \H) \neq \stb(\G \cup \H)$.
Let $B = A(\F) \cup A(\G) \cup A(\H)$ and $\H' = (B \cup \{a\},\{(a,b),(b,a) \mid b \in B\})$.
It is easy to see that $\stb(\F \cup \H') = \stb(\F \cup \H) \cup \{\{a\}\}$ and
$\stb(\G \cup \H') = \stb(\G \cup \H) \cup \{\{a\}\}$.
Since now both $\stb(\F \cup \H') \neq \emptyset$ and $\stb(\G \cup \H') \neq \emptyset$
it holds that $\stb(\F \cup \H') = \stg(\F \cup \H')$ and $\stb(\G \cup \H') = \stg(\G \cup \H')$.
Hence $\sigma(\F \cup \H') \neq \sigma(\F \cup \H')$,
showing that $\F \not\equiv^\stb_E \G$.
\end{proof}

The following theorems can be shown in a similar manner.
 
\begin{theorem} \label{the:admfirst}
Given a semantics $\sigma$ which is $+\mp$-verifiable and \interm{\rho}{\ad} with $\rho\in\{\semi,\id,\eg\}$,
it holds that
\[
\F^{k(\ad)} = \G^{k(\ad)} \ToT \F\equiv^\sigma_E\G.
\]
\end{theorem}

Remember that complete semantics is a \interm{\semi}{\adm} semantics. Furthermore, it is not characterizable by the admissible kernel as already observed in \cite{strong}. Consequently, complete semantics is not $+\mp$-verifiable (as we have shown in Example~\ref{ex:exact_verify} with considerable effort).

\begin{theorem}
\label{the:grdfirst}
Given a semantics $\sigma$ which is $-\pm$-verifiable and \interm{\gr}{\strad},
it holds that
\[
\F^{k(\gr)} = \G^{k(\gr)} \ToT \F\equiv^\sigma_E\G.
\]
\end{theorem}

\section{Conclusions}

In this work we have %
contributed to the analysis and comparison of abstract argumentation semantics.
The main idea of our approach
is to provide a novel categorization 
in terms of the amount of information
required for testing whether a set of arguments is an extension of a certain
semantics.
The resulting notion of verifiability classes allows us to categorize any new semantics (given it is ``rational'')
with respect to the information needed and compare it to other semantics.
Thus our work is 
in the tradition of the principle-based evaluation due to 
\citeauthor{BaroniG07a} 
\shortcite{BaroniG07a} and 
paves the way for a more general view on argumentation semantics, their 
common features, and their inherent differences.

Using our notion of verifiability, we were able to show kernel-compatibility
for certain intermediate semantics.
Concerning concrete semantics, our results yield the following observation:
While preferred, semi-stable, ideal and eager semantics coincide w.r.t.\
strong equivalence, verifiability of these semantics differs.
In fact, preferred and ideal semantics manage to be verifiable with strictly less information.

For future work we envisage an extension of the
notion of verifiability classes in order to categorize
semantics not captured by the approach followed in this paper,
such as $\cft$ \cite{BaroniGG05a}.

\end{document}